\documentclass{article}

\usepackage{microtype}
\usepackage{graphicx}
\usepackage{subfigure}
\usepackage{subfig}
\usepackage{booktabs} 
\usepackage{amsmath}
\usepackage{amssymb}
\usepackage{amsthm}
\usepackage{multirow}
 
\theoremstyle{definition}
\newtheorem{definition}{Definition}
\newtheorem{lemma}{Lemma}
\newtheorem{thm}{Theorem}
\usepackage{multicol} 
\usepackage{float}

\usepackage{color}

\usepackage[accepted]{icml2020}

\icmltitlerunning{}

\begin{document}

\twocolumn[
\icmltitle{AT-GAN: An Adversarial Generator Model \\ for Non-constrained Adversarial Examples} 



\icmlsetsymbol{equal}{*}

\begin{icmlauthorlist}
\icmlauthor{Xiaosen Wang}{equal,hust}
\icmlauthor{Kun He}{equal,hust}
\icmlauthor{Chuanbiao Song}{hust} 
\icmlauthor{Liwei Wang}{pku}
\icmlauthor{John Hopcroft}{corn} 
\end{icmlauthorlist}

\icmlaffiliation{hust}{School of Computer Science and Technology, Huazhong University of Science and Technology, Wuhan, China}
\icmlaffiliation{pku}{School of Electronics Engineering and Computer Sciences, Peking University, Peking, China}
\icmlaffiliation{corn}{Department of Computer Science, Cornell University, NY, USA}

\icmlcorrespondingauthor{Kun He}{brooklet60@hust.edu.cn}
\icmlcorrespondingauthor{Xiaosen Wang}{xiaosen@hust.edu.cn}

\icmlkeywords{Machine Learning, ICML}

\vskip 0.3in
]



\printAffiliationsAndNotice{\icmlEqualContribution} 

\begin{abstract}
Despite the rapid development of adversarial machine learning, most adversarial attack and defense researches mainly focus on the perturbation-based adversarial examples, which is constrained by the input images. In comparison with existing works, we propose non-constrained adversarial examples, which are generated entirely from scratch without any constraint on the input. Unlike perturbation-based attacks, or the so-called unrestricted adversarial attack which is still constrained by the input noise, we aim to learn the distribution of adversarial examples to generate non-constrained but semantically meaningful adversarial examples. Following this spirit, we propose a novel attack framework called AT-GAN (Adversarial Transfer on Generative Adversarial Net). Specifically, we first develop a normal GAN model to learn the distribution of benign data, and then transfer the pre-trained GAN model to estimate the distribution of adversarial examples for the target model. In this way, AT-GAN can learn the distribution of adversarial examples that is very close to the distribution of real data. To our knowledge, this is the first work of building an adversarial generator model that could produce adversarial examples directly from any input noise.
Extensive experiments and visualizations show that the proposed AT-GAN can very efficiently generate diverse adversarial examples that are more realistic to human perception. In addition, AT-GAN yields higher attack success rates against adversarially trained models under white-box attack setting and exhibits moderate transferability against black-box models. 
\end{abstract}

\section{Introduction}
\label{sec:intro}

Deep Neural Networks (DNNs) have exhibited impressive performance on various computer vision tasks~\citep{Krizhevsky2012AlexNet, He2016Resnet}. However, DNNs are also found vulnerable to adversarial examples~\citep{Szegedy2014Intriguing}, which has raised serious concerns on the safety of deep learning models.
For instance, \citet{Eykholt2018Robust} and  \citet{Cao2019Adv} have found that adversarial examples can mislead the self-driving cars to make wrong decisions on traffic signs and obstacles. 

In recent years, numerous works of exploring adversarial examples have been developed, including adversarial attack \citep{Goodfellow2015Explaining, Carlini2017Towards}, adversarial defense \citep{Goodfellow2015Explaining, Kurakin2017Adv, Song2019Improving} and the properties of adversarial examples \citep{He2018Decision, Shamir2019Explanation}. However, most studies mainly focus on the \textit{perturbation-based adversarial examples} constrained by the input images. In contrast, \citet{Song2018Constructing} propose to generate \textit{unrestricted adversarial examples} using Generative Adversarial Net (GAN) \citep{Goodfellow2014GANs} by searching adversarial noise around the input noise. However, their method is still constrained by the input noise, and time-consuming.


In this paper, we propose a novel attack framework called AT-GAN (Adversarial Transfer on Generative Adversarial Net), which aims to learn the distribution of adversarial examples and generate a new type of adversarial examples, called \textit{non-constrained adversarial examples}. Unlike traditional perturbation-based attacks that search adversarial examples around the input image, or the unrestricted adversarial attack~\citep{Song2018Constructing} that searches adversarial noise around the input noise, AT-GAN is an adversarial generator model that could produce semantically meaningful adversarial examples directly from any input noise. This to our knowledge is the first work of this category. 

Specifically, we first develop a normal GAN model to learn the distribution of benign data, then transfer the pre-trained GAN model to estimate the distribution of adversarial examples for the target model. Note that once our generator is transferred from generating normal images to adversarial images, it can directly generate non-constrained adversarial examples for any input random noise, leading to high diversity and efficiency. 
In general, AT-GAN is not only a new way of generating adversarial examples, but also a useful way to explore the distribution of adversarial examples. 


To evaluate the effectiveness of AT-GAN, we develop AT-GAN on three benchmark datasets, MNIST, Fashion-MNIST and CelebA,
 and apply typical defense methods 
to compare AT-GAN with other perturbation-based or unrestricted attack methods.
Empirical results show that the non-constrained adversarial examples generated by AT-GAN yields higher attack success rates, and state-of-the-art adversarially trained models exhibit little robustness against AT-GAN, indicating the high diversity of the generated examples. 
Besides, AT-GAN is more efficient than the gradient-based or query-based adversarial attacks in generating adversarial instances. 

\section{Preliminaries}
\label{sec:background}
In this section, we first give the definitions of different types of adversarial examples, then introduce Generative Adversarial Net (GAN) that will be used in our method, and provide more details of unrestricted adversarial examples which are most related to our work. The related works on typical adversarial attack and defense methods are introduced in Appendix \ref{app:related_work}.

Let \(\mathcal{X}\) be the set of the legal images, \(\mathcal{Y} \in \mathbb{R}\) be the output label space and \(p_z\) be an arbitrary probability distribution, such as Gaussian distribution. The classifier \(f:\mathcal{X} \rightarrow \mathcal{Y}\) takes an image \(x\) and predicts its label \(f(x)\). Suppose \(p_x\) and \(p_{adv}\) are the distributions of benign and adversarial examples respectively.

\subsection{Adversarial Examples}
\label{sec:adv}
Assume we have an oracle \(o:\mathcal{X} \rightarrow \mathcal{Y}\), which could always predict the correct label for any input \(x \in \mathcal{X}\), we can define several types of adversarial examples as follows:

For perturbation-based adversarial examples, tiny perturbations are added to input images, which are undetectable by human but can cause the target classifier to make wrong prediction.

\begin{definition}
\textbf{\textit{Perturbation-based Adversarial Examples.}}
\label{def:per_adv}
Given a subset (trainset or testset) \(\mathcal{T} \subset \mathcal{X}\) and a small constant \(\epsilon > 0\), the perturbation-based adversarial examples can be defined as: \(\mathcal{A}_p = \{x_{adv} \in \mathcal{X} | \exists x \in \mathcal{T}, \|x-x_{adv}\|_p < \epsilon \wedge f(x_{adv}) \neq o(x_{adv}) = f(x) = o(x) \}\).
\end{definition}

\citet{Song2018Constructing} define a new type of adversarial examples, called unrestricted adversarial examples, which add adversarial perturbation to the input noise of a mapping like GAN so that the output of the perturbed noise is an adversary for the target classifier.

\begin{definition}
\textbf{\textit{Unrestricted Adversarial Examples.}}
\label{def:unres_adv}
Given a mapping \(G\) from \(z \sim p_z\) to \(G(z) \sim p_\theta\), where \(p_\theta\) is an approximation distribution of \(p_x\), and a small constant $\epsilon > 0$,  unrestricted adversarial examples  can be defined as: $\mathcal{A}_u = \{G(z^*) \in \mathcal{X} | \exists z\sim p_z, z^* \sim p_z, \|z - z^*\|_p < \epsilon \wedge f(G(z^*)) \neq o(G(z^*)) = f(G(z)) = o(G(z))\}$.
\end{definition}

In this work, we define a new type of adversarial examples, called non-constrained adversarial examples, in which we train a mapping like GAN to learn the distribution of adversarial examples so as to generate adversaries from any noise.

\begin{definition}
\textbf{\textit{Non-constrained Adversarial Examples.}}
Given a mapping \(G^*\) from \(z \sim p_z\) to \(G^*(z) \sim q_\theta\), where $q_\theta$ is an  approximation distribution of $p_{adv}$, the non-constrained adversarial examples can be defined as: $\mathcal{A}_n = \{G^*(z) \in \mathcal{X} | o(G^*(z) \neq f(G^*(z))\}$.
\end{definition}

In summary, perturbation-based adversarial examples are based on perturbing an image $x \in \mathcal{X}$, and unrestricted adversarial examples are based on perturbing an input noise $z\sim p_z$ for an existing mapping $G$. In contrast, non-constrained adversarial examples are more generalized that we learn a mapping \(G^*\) such that for any input noise sampled from distribution $p_z$, the output can fool the classifier. It is clear that $\mathcal{A}_p \subset \mathcal{A}_u \subset \mathcal{A}_n$ from the above definitions.

\subsection{Generative Adversarial Net}

Generative Adversarial Net (GAN) \citep{Goodfellow2014GANs} consists of two neural networks, \(G\) and \(D\), trained in opposition to each other. The generator \(G\) is optimized to estimate the data distribution and the discriminator \(D\) aims to distinguish fake samples from \(G\) and real samples from the training data. The objective of \(D\) and \(G\) can be formalized as a min-max value function \(V(G, D)\):
\begin{equation*}
    \label{eq:gan}
    \begin{split}
        \min_G \max_D V(G, D) = \mathbb{E}_{x\sim p_{x}}[\log D(x)] + \\
        \mathbb{E}_{z\sim p_z} [\log (1-D(G(z)))].    
    \end{split}
\end{equation*}

Deep Convolutional Generative Adversarial Net (DCGAN) \citep{Radford2016DCGAN} is the convolutional version of GAN, which implements GAN with convolutional networks and stabilizes the training process. Auxiliary Classifier GAN (AC-GAN) \citep{Odena2017ACGAN} is another variant that extends GAN with some conditions by an extra classifier \(C\). 
The objective function of AC-GAN can be formalized as follows:
\begin{equation*}
\begin{split}
\min_G \max_D \min_C &V(G, D, C) = \mathbb{E}_{x\sim p_{x}}[\log D(x)]\\
& + \mathbb{E}_{z\sim p_z}[\log (1-D(G(z, y)))]\\ 
& + \mathbb{E}_{x\sim p_{x}}[\log (1 - C(x, y))]\\ 
& + \mathbb{E}_{z\sim p_z}[\log (1 -C(G(z,y), y))].
\end{split}
\end{equation*}

To make GAN more trainable in practice, \citet{Arjovsky2017WGAN} propose Wasserstein GAN (WGAN) that uses Wassertein distance so that the loss function has more desirable properties. \citet{Gulrajani2017WGANGP} introduce WGAN with gradient penalty (WGAN\_GP) which outperforms WGAN in practice. Its objective function is formulated as:
\begin{equation*}
    \begin{split}
        &\min_G \max_D V(D,G) = \mathbb{E}_{x \sim p_{x}}[D(x)] - \\ 
        &\mathbb{E}_{z \sim p_z}[D(G(z))] - \lambda \mathbb{E}_{\hat{x} \sim p_{\hat{x}}}[(\|\nabla_{\hat{x}}D(\hat{x})\|_2 - 1)^2],        
    \end{split}
\end{equation*}
where $p_{\hat{x}}$ is uniformly sampling along straight lines between
pairs of points sampled from the data distribution $p_x$ and the generator distribution $p_g$.

\subsection{Unrestricted Adversarial Examples}
\citet{Song2018Constructing} propose to construct Unrestricted Adversarial Examples using AC-GAN. Specifically, given an arbitrary noise vector \(z^0\), their method aims to search a noise input \(z^*\) in the \(\epsilon\)-neighborhood of \(z^0\) for AC-GAN so as to produce an adversarial example \(G(z^*, y_{s})\) for the target model \(f\), whereas \(G(z^*, y_{s})\) can be still classified as \(y_{s}\) for the extra classifier \(C\).
The objective function can be written as:
\begin{equation*}
    \begin{split}
        z^* =& \mathop{\arg\min}_{z} \{ \frac{\lambda_1 }{m}\sum_{i=1}^m \max(|z_i-z_i^0|-\epsilon, 0) \\
        &-\lambda_2 \log C(y_s|G(z^*, y_{s})) \\
        &-max_{y \neq y_s} \log f(y|G(z^*, y_{s}))\} \ ,    
    \end{split}
\end{equation*}

where \(C(y|x)\) (or \(f(y|x)\)) denotes the confidence of prediction label \(y\) for the given input \(x\) on classifier \(C\) (or \(f\)). 

Though \(G(z^*, y_{s})\) can be classified correctly by \(C\), it cannot assure that \(G(z^*, y_{s})\) is a realistic image for human eye~\citep{Nguyen2015Unrecognizable}. Therefore, their method needs human evaluation for the generated adversarial examples.

\section{The Proposed AT-GAN}
\label{sec:method}
In this section, we first introduce the estimation on the distribution of adversarial examples, then propose the AT-GAN framework 
to generate non-constrained adversarial examples, and provide further analysis that AT-GAN can learn the distribution of adversarial examples in the end.

\subsection{Estimating the Adversarial Distribution}

\begin{figure}[htb]
    \begin{center}
        \includegraphics[scale=0.69]{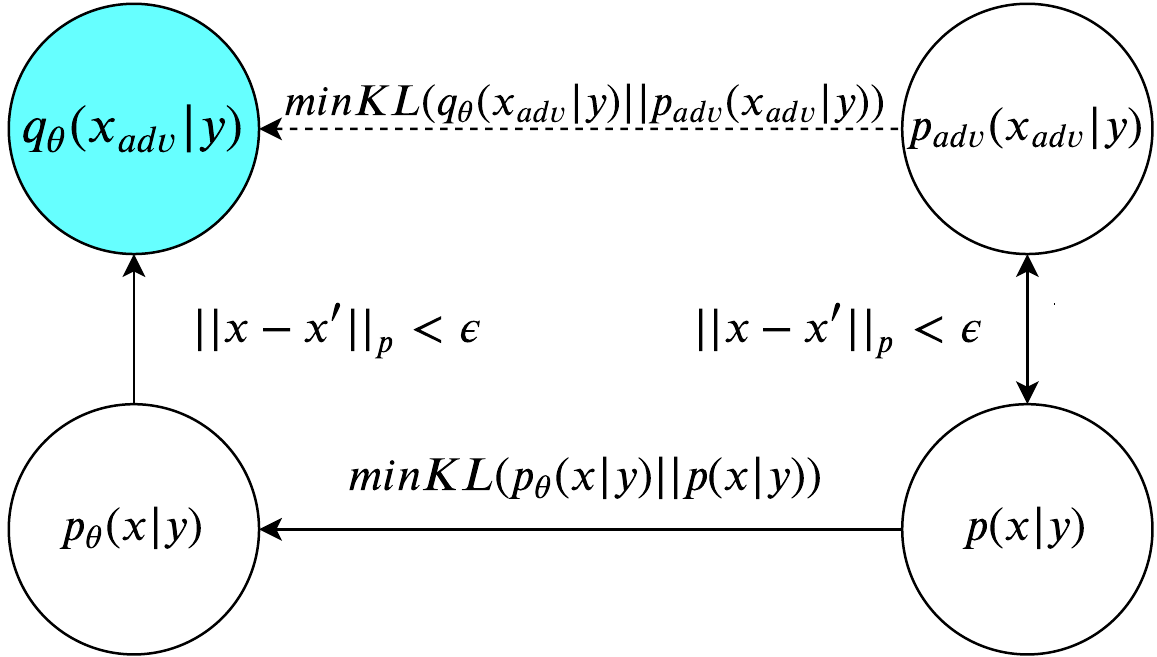}
        \caption{Estimating the distribution of adversarial examples $q_\theta$ in two stages: 1) estimate the distribution of benign data $p_\theta$. 2) transfer $p_\theta$ to estimate $q_\theta$.}
        \label{fig:appro_dis}
    \end{center}
\end{figure}

\label{sec:dis}
In order to generate non-constrained adversarial examples, we need to estimate the distribution of adversarial examples $p_{adv}(x_{adv}|y)$. Given the parameterized estimated distribution $q_{\theta}(x)$, we can define the estimation problem as
\begin{equation}
    \label{eq:dis}
    q_{\theta^*}(x_{adv}|y) = \mathop{\arg\min}_{\theta \in \Omega} KL(q_\theta(x_{adv}|y)\| p_{adv}(x_{adv}|y)) ,
\end{equation}
where $\theta$ indicates trainable parameters and $\Omega$ is the parameter space.

It is hard to calculate Eq. \ref{eq:dis} directly as $p_{adv}(x_{adv}|y)$ is unknown. Inspired by the perturbation-based adversarial examples, as shown in Figure \ref{fig:appro_dis}, we postulate that for each adversarial example $x_{adv}$, there exists some benign examples $x$ where $\|x-x_{adv}\|_p < \epsilon$. 
In other words, $p_{adv}(x_{adv}|y)$ is close to $p(x|y)$ to some extent. By Bayes' theorem, we  have
\begin{equation*}
    p(x|y) = \frac{p(y|x) \cdot p(x)}{p(y)}.
\end{equation*} 

Thus, we can approximately solve Eq. \ref{eq:dis} in two stages: 1) Fit the distribution of benign data $p_\theta$. 2) Transfer $p_\theta$ to estimate the distribution of adversarial examples $q_\theta$.

Specifically, we propose a new adversarial attack model called AT-GAN to learn the distribution of adversarial examples. The overall architecture of AT-GAN is illustrated in Figure \ref{fig:attack_gan}. Corresponding to the above two training stages, we first train the GAN model to get a generator \(G_{original}\) to learn $p_\theta$, then we transfer \(G_{original}\) to attack \(f_{target}\) for the learning of $q_\theta$.

\begin{figure*}[htbp]
\includegraphics[scale=0.75]{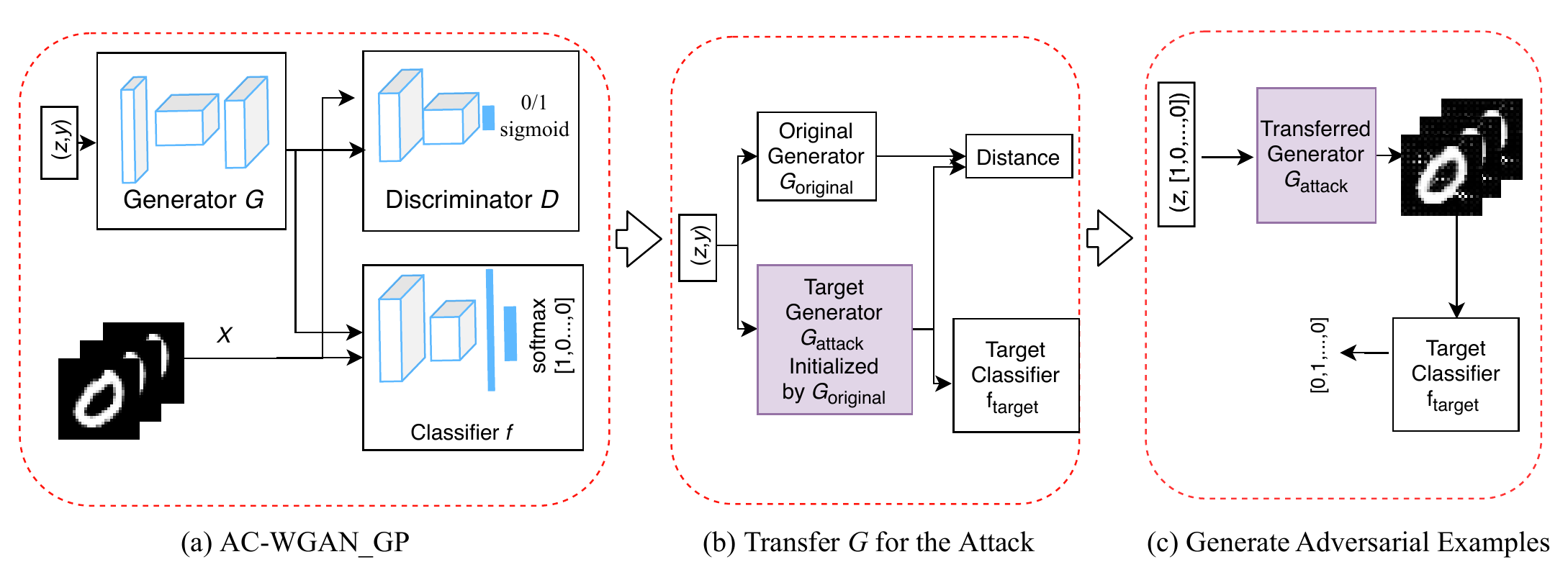}
\caption{The architecture of AT-GAN. The first training stage of AT-GAN is similar to that of AC-WGAN\_GP. After \(G\) is trained, we regard \(G\) as the original model \(G_{original}\) and transfer \(G_{original}\) to train the adversarial generator \(G_{attack}\) to fool the target classifier. After the second stage of training, AT-GAN can generate adversarial examples by  \(G_{attack}\).}
\label{fig:attack_gan}
\end{figure*}

\subsection{Training the Original Generator}
\label{Generator}

Figure \ref{fig:attack_gan} (a) illustrates the overall architecture of AC-WGAN\_GP that we used as the normal GAN. AC-WGAN\_GP is the combination of AC-GAN \citep{Odena2017ACGAN} and WGAN\_GP \citep{Gulrajani2017WGANGP}, composed by three neural networks: a generator \(G\), a discriminator \(D\) and a classifier \(f\). The generator \(G\) takes a random noise \(z\) and a lable \(y\) as the inputs and generates an image \(G(z, y)\). It aims to generate an image \(G(z, y)\) that is indistinguishable to discriminator \(D\) and makes the classifier \(f\) to output label \(y\). The loss function of \(G\) can be formulated as:
\begin{equation*}
\begin{split}
L_G(z,y) = &\mathbb{E}_{z\sim p_z(z)} [H(f(G(z,y)), y)]\\
 &- \mathbb{E}_{z \sim p_z(z)}[D(G(z))].
\end{split}
\end{equation*}
Here \(H(a, b)\) is the entropy between \(a\) and \(b\).
The discriminator \(D\) takes the training data \(x\) or the generated data \(G(z,y)\) as the input and tries to distinguish them. The loss function of \(D\) with gradient penalty for samples 
\(\hat{x} \sim p_ {\hat{x}}\) can be formulated as:
\begin{equation*}
\begin{split}
    L_D(x,z,y) = &-\mathbb{E}_{x\sim p_{data}(x)} [D(x)] \\
    &+ \mathbb{E}_{z\sim p_z(z)} [D(G(z,y))] \\
    &+ \lambda \mathbb{E}_{\hat{x} \sim p_{\hat{x}}(\hat{x})}[(\|\nabla_{\hat{x}}D(\hat{x})\|_2 - 1)^2].
\end{split}
\end{equation*}
The classifier \(f\) takes the training data \(x\) or the generated data \(G(z,y)\) as the input and predicts the corresponding label. There is no difference from other classifiers and the loss function is:
\begin{equation*}
\begin{split}
    L_f(x, y) = &\mathbb{E}_{x\sim p_{data}(x)} [H(f(x), y)] + \\
    &\mathbb{E}_{z\sim p_z(z)} [H(f(G(z,y)), y)].
\end{split}
\end{equation*}
Note that the goal of this stage is to train a generator \(G\) which could output realistic samples and estimate the distribution of real data properly so that we could later on transfer the generator \(G\) to estimate the distribution of adversarial examples. So one could train a generator \(G\) using other methods as long as the generator \(G\) could learn a good distribution of the real data.

\subsection{Transferring the Generator for Attack}
After the original generator \(G_{original}\) is trained, we transfer the generator \(G_{original}\) to learn the distribution of adversarial examples in order to attack the target model.
As illustrated in Figure \ref{fig:attack_gan}b, there are three neural networks, including the original generator \(G_{original}\), the attack generator \(G_{attack}\) to be transferred that has the same weights as \(G_{original}\) in the beginning, and the classifier \(f_{target}\) to be attacked. The goal of the second stage can be described as:
\begin{gather}
    G_{attack}^* = \mathop{\arg\min}_{G_{attack}} \ ||G_{original}(z,y) - G_{attack}(z,y)||_p \nonumber\\
        s.\ t. \quad f_{target}(G(z, y))=y_t \neq \ y ,
    \label{equivalent_problem}
\end{gather}
where \(y_t\) denotes the target adversarial label, \(\|\cdot\|_p\) denotes the \(\ell_p\) norm and we focus on \(p=2\) in this work.  

To optimize Eq. \ref{equivalent_problem}, we construct the loss function by \(L_a\) and \(L_d\), where \(L_a\) aims to assure that \(f_{target}\) yields the target label \(y_t\):
\begin{equation}
L_a(z,y) = \mathbb{E}_{z\sim p_z}[H(f_{target}(G_{attack}(z,y)), y_t)], 
\end{equation}
and \(L_d\) aims to assure that the adversarial generator \(G_{attack}\) generates realistic examples:
\begin{equation}
\label{Ld}
L_d(z,y) = \mathbb{E}_{z\sim p_z}[||G_{original}(z,y) + \rho - G_{attack}(z,y)||_p].
\end{equation}
Here \(\rho\) is a random noise constrained by both \(l_0\) and \(l_\infty\) norm.

The objective function for transferring \(G_{original}\) to \(G_{attack}\) can be formulated as:
\begin{equation}
L(z,y) = \alpha L_a(z,y) + \beta L_d(z,y),
\end{equation}
where \(\alpha\) and \(\beta\) are hyperparameters to control the training process. Note that in the case that \(\alpha = 1\) and \(\beta \to \infty\), the objective function is similar to that of the perturbation-based attacks \citep{Goodfellow2015Explaining, Florian2018Ensemble,Madry2018Towards}.

For the untargeted attack, we can replace \(y_t\) in \(L_a\) with \(\max_{y \neq y_s} C(G_{attack}(z,y_s))\), where \(C(\cdot)\) is the logits of the target classifier \(f_{target}\).

\subsection{Theoretical Analysis of AT-GAN}
\label{sec:ana}
In this subsection, we provide theoretical analysis why AT-GAN can generate as realistic and diverse non-constrained adversarial examples as real data. We will prove that under ideal condition, AT-GAN can estimate the distribution of adversarial examples which is close to that of real data as described in Section \ref{sec:dis}. 

Suppose $p_{data}$ is the distribution of real data, $p_g$ and $p_a$ are the distribution learned by the generator of AC-WGAN\_GP and AT-GAN respectively. For the optimization of Eq. \ref{equivalent_problem}, \(L_d\) aims to constrain the image generated by \(G_{attack}\) in the \(\epsilon\)-neighborhood of \(G_{original}\). We prove that under the ideal condition that \(L_d\) guarantees \(G_{attack}(z,y)\) to be close enough to \(G_{original}(z,y)\) for any input noise \(z\), the distribution of AT-GAN almost coincides the distribution of AC-WGAN\_GP. Formally, we state our result for the two distributions as follows.

\begin{thm}
Suppose $\max_{z,y} L_d(z,y) < \epsilon$, we have $KL(p_a \| p_g) \to 0$ when $\epsilon \to 0$.
\label{thm:atgan}
\end{thm}


The proof of Theorem \ref{thm:atgan} is in Appendix \ref{app:exp}. \citet{Samangouei2018DefenseGAN} prove that the global optimum of WGAN is \(p_g=p_{data}\) and we can show that the optimum of AC-WGAN\_GP has the same property. We can formalize the property as follows:

\begin{thm}\label{thm:OPTIMUMWGAN}
The global minimum of the virtual training of AC-WGAN\_GP is achieved if and only if \(p_g = p_{data}\).
\end{thm}

The proof of Theorem \ref{thm:OPTIMUMWGAN} is in Appendix \ref{app:exp}. According to Theorem \ref{thm:atgan} and \ref{thm:OPTIMUMWGAN}, under the ideal condition, we conclude \(p_a \approx p_g = p_{data}\), which indicates that the distribution of adversarial examples learned by AT-GAN is very close to that of real data as discussed in Section \ref{sec:dis}, so that the non-constrained adversarial instances are as realistic and diverse as the real data.

\section{Experiments}
\label{sec:exp}
To validate the effectiveness of AT-GAN, we empirically evaluate our non-constrained adversarial examples on various datasets, and demonstrate that AT-GAN yields higher attack success rates against adversarial training and achieves moderate transferability with higher efficiency. Besides, AT-GAN can learn a distribution of adversarial examples which is close to the real data distribution, and generate realistic and diverse adversarial examples.
\subsection{Experimental Setup}
\textbf{Datasets.} We consider three standard datasets, namely MNIST \citep{LeCun1998MNIST}, Fashion-MNIST \citep{Han2017Fashion-MNIST} and CelebA \citep{Liu2015CelebA}. Specifically, MNIST is a dataset of hand written digit number from \(0\) to \(9\). Fashion-MNIST is similar to MNIST with \(10\) categories of fashion clothes.
CelebA is a dataset with more than \(200,000\) celebrity images, and we group them according to female/male and focus on gender classification as in \citet{Song2018Constructing}.
For all datasets, we normalize the pixel values into range \([0,1]\).

\textbf{Baselines.} We compare the proposed AT-GAN with typical perturbation-based adversarial attacks, i.e., FGSM \citep{Goodfellow2015Explaining}, PGD \citep{Madry2018Towards}, R+FGSM \citep{Florian2018Ensemble} and unrestricted adversarial attack \citep{Song2018Constructing}. 

\textbf{Models.} For MNIST and Fashion-MNIST, we adopt four models used in \citet{Florian2018Ensemble}, denoted as \textsf{Model A} to \textsf{D}.
For CelebA, we consider three models, namely CNN, VGG16 and ResNet. More details about the architectures are provided in Appendix \ref{app:exp}.

\textbf{Evaluation Setup.} 
To evaluate the performance of these attacks, we consider normal training and existing strong defenses, namely adversarial training \citep{Goodfellow2015Explaining}, ensemble adversarial training \citep{Florian2018Ensemble} and iterative adversarial training~\citep{Madry2018Towards}.

All experiments are conducted on a single Titan X GPU. The hyper-parameters used in the experiments are described in Appendix \ref{app:exp}.

\subsection{Evaluation Results}
\subsubsection{Comparison on Attack Success Rate}
\begin{table*}[htb]
\caption{Attack success rate (\%) of adversarial examples generated by AT-GAN and the baselines attacks under white-box attack setting against models by normal training and various adversarial training methods. For each model, the highest attack success rate is highlighted in \textbf{bold}. Nor.: Normal training, Adv.: Adversarial training, Ens.: Ensemble adversarial training, Iter. Adv.: Iterative adversarial training.}
\vspace{0.2em}
\label{tab:adv_res}
\centering

\subtable[Comparison of AT-GAN and other attacks on MNIST.]{
\scalebox{0.84}{
  \begin{tabular}{ccccccccccccccccc}
    \toprule
    \multirow{2}{*}{Attack} 	& \multicolumn{4}{c}{Model A}	& \multicolumn{4}{c}{Model B}	& \multicolumn{4}{c}{Model C} & \multicolumn{4}{c}{Model D}\\
    \cmidrule(r){2-5}     \cmidrule(r){6-9}     \cmidrule(r){10-13} \cmidrule(r){14-17}
   	&  Nor. & Adv. & Ens. & Iter. & Nor. & Adv. & Ens. & Iter. & Nor. & Adv. & Ens. & Iter.  & Nor. & Adv. & Ens. & Iter.\\
    \midrule
    Clean Error (\%) & ~~0.9 & ~~0.9 & ~~ 0.8 & ~~ 0.9 & ~~1.0 & ~~0.9 & ~~1.0 & ~~0.9 & ~~1.0  & ~~1.0  & ~~1.0  & ~~1.0 & ~~2.7 & ~~2.6 & ~~2.6 & ~~1.5 \\
    \midrule
	FGSM 	& 68.8 & ~~2.9 & ~~8.7 & ~~4.0 & 88.0 & ~~9.5 & 18.0 & ~~9.2 & 70.7 & ~~4.3 & ~~7.8 & ~~4.7 & 89.6 & 23.5 & 34.6 & 26.7\\
	PGD 	& \textbf{100.0} & 92.6 & 85.1 & ~~5.9 & \textbf{100.0} & 42.4 & 98.2 & 36.2 & \textbf{100.0} & 76.9 & 96.4 & ~~9.6 & 91.7 & 96.8 & \textbf{99.5} & 81.1\\
	R+FGSM 	& 77.9 & 28.4 & 20.3 & ~~2.6 & 96.5 & ~~7.1 & 42.1 & ~~4.6 & 81.6 & ~~7.2 & 19.9 & ~~2.9 & 93.8 & 76.2 & 51.6 & 25.3\\
	Song's 	& 82.0 & 70.5 & 75.0 & 84.6 & 76.9 & 65.0 & 72.0 & 80.7 & 74.2 & 75.6 & 72.6 & 87.8 & 67.7 & 43.6 & 56.3 & 44.5\\
	AT-GAN 	& 98.7 & \textbf{97.5} & \textbf{96.7}  & \textbf{91.4} & 99.5 & \textbf{97.7} & \textbf{99.3} & \textbf{95.6} & 99.3 & \textbf{95.8} & \textbf{96.9} & \textbf{90.0} & \textbf{99.9} & \textbf{99.9} & \textbf{99.5} & \textbf{99.7}\\
    \bottomrule
    
  \end{tabular}
}}
\vspace{0.2em} 
\subtable[Comparison of AT-GAN and other attacks on Fashion-MNIST.]{
 
\scalebox{0.86}{
  \begin{tabular}{ccccccccccccccccc}
    \toprule
    \multirow{2}{*}{Attack} 	& \multicolumn{4}{c}{Model A}	& \multicolumn{4}{c}{Model B}	& \multicolumn{4}{c}{Model C} & \multicolumn{4}{c}{Model D}\\
    \cmidrule(r){2-5}     \cmidrule(r){6-9}     \cmidrule(r){10-13} \cmidrule(r){14-17}
   	&  Nor. & Adv. & Ens. & Iter. & Nor. & Adv. & Ens. & Iter. & Nor. & Adv. & Ens. & Iter.  & Nor. & Adv. & Ens. & Iter.\\
    \midrule
    Clean Error (\%) & ~~8.2 & ~~8.3 & ~~8.0 & 10.5 & ~~9.8 & ~~9.3 & 10.4 & ~~9.3 & ~~7.8 & ~~8.1 & ~~7.9 & 10.0 & 14.8 & 13.7 & 14.5 & 10.6 \\ 
    \midrule
	FGSM 	& 82.7 & 14.6 & 36.0 & 23.2 & 82.3 & 24.2 & 30.0 & 23.8 & 82.7 & 11.0 & 47.8 & 22.3 & 77.2 & 33.8 & 47.5 & 22.9\\
	PGD 	& \textbf{99.9} & 82.6 & 90.8 & 30.2 & 96.0 & 69.9 & \textbf{95.3} & 34.7 & \textbf{100.0} & 88.3 & 99.9 & 28.8 & 85.4 & 54.7 & 76.8 & 30.2\\
	R+FGSM 	& 95.5 & 63.3 & 68.4 & 37.8 & 90.9 & 74.2 & 81.3 & 28.2 & 98.4 & 76.5 & 71.3 & 31.2 & 88.3 & 45.3 & 62.1 & 33.6\\
	Song's 	& 93.1 & \textbf{96.7} & 95.3 & 82.1 & 80.3 & 88.5 & 92.2 & 80.0 & 96.5 & \textbf{95.6} & \textbf{96.6} & 83.2 & 66.8 & 49.2 & 63.1 & 84.6 \\
	AT-GAN 	& 96.1 & 92.7 & \textbf{95.4} & \textbf{93.5} & \textbf{98.5} & \textbf{91.1} & 93.6 & \textbf{91.6} & 98.0 & 88.9 & 93.9 & \textbf{91.6} & \textbf{99.9} & \textbf{99.4} & \textbf{99.1} & \textbf{93.1}\\
    \bottomrule
    
  \end{tabular}
}}
\vspace{0.2em} 
\subtable[Comparison of AT-GAN and other attacks on CelebA.]{
  \begin{tabular}{ccccccccccccccc}
    \toprule
    \multirow{2}{*}{Attack} 	& \multicolumn{4}{c}{CNN}	& \multicolumn{4}{c}{VGG}	& \multicolumn{4}{c}{ResNet} \\
    \cmidrule(r){2-5}     \cmidrule(r){6-9}     \cmidrule(r){10-13} 
   	&  Nor. & Adv. & Ens. & Iter. & Nor. & Adv. & Ens. & Iter. & Nor. & Adv. & Ens. & Iter. \\
    \midrule
    Clean Error (\%) & ~~2.7 & ~~2.7 & ~~3.2 & ~~2.4 & ~~2.1 & ~~3.0 & ~~2.9 & ~~2.4 & ~~2.3 & ~~3.0 & ~~3.0 & ~~2.5\\
    \midrule
	FGSM 	& 81.2 & 11.8 & 14.7 & ~~9.5 & 76.7 & 10.6 & 16.2 & ~~8.9 & 98.7 & ~~9.8 & 12.3 & ~~9.5\\
	PGD 	& 97.3 & 16.4 & 22.6 & 11.4 & 87.9 & 14.6 & 26.3 & 10.7 & \textbf{100.0} & 10.9 & 15.1 & 10.5\\
	R+FGSM 	& 68.7 & ~~9.5 & 11.3 & ~~7.9 & 68.4 & ~~8.7 & 13.2 & ~~7.3 & 97.5 & ~~7.9 & ~~9.5 & ~~7.8\\
	Song's 	& 90.6 & 83.4 & 85.7 & 89.8 & \textbf{98.7} & 87.5 & 95.7 & 81.6 & 99.2 & 93.4 & 91.0 & 90.6\\
    	AT-GAN 	& \textbf{97.5} & \textbf{98.9} & \textbf{95.9} & \textbf{99.6} & 97.1 & \textbf{96.7} & \textbf{95.8} & \textbf{97.8} & 98.4 & \textbf{98.5} & \textbf{97.3} & \textbf{98.5}\\
    \bottomrule
  \end{tabular}}
\end{table*}
To validate the efficacy, we compare AT-GAN with other baseline attack methods under white-box setting. Since \citet{Athalye2018Obfuscated} have shown that the currently most effective defense method is adversarial training, we consider adversarially trained models as the defense models. 
The attack success rates are reported in Table \ref{tab:adv_res}.

On MNIST, AT-GAN is the best attack method against all the defense models. As for \textsf{Model A}, \textsf{B} and \textsf{C} by normal training, AT-GAN gains the second highest attack success rates over \(98\%\). On Fashion-MNIST, AT-GAN achieve the highest attack success rate on average. On CelebA, AT-GAN achieves the best attack performance almost on all models. 
The only exceptions are that \citet{Song2018Constructing} achieves the highest attack rate on VGG, and PGD achieves the highest attack rate on ResNet by normal training respectively. Under the two cases, the results of AT-GAN  are  still close to the best results.

On average, AT-GAN achieves the highest attack success rate on all defense models. As AT-GAN aims to estimate the distribution of adversarial examples, adversarial training with some specific attacks has little robustness against AT-GAN, raising a new security issue for the development of more generalized adversarially training models. 

\subsubsection{Comparison on Attack Efficiency}
There are many scenarios where one needs large amount of adversarial examples, such as adversarial training or exploring the properties of adversarial examples.
The efficiency for adversarial attacks to generate adversarial examples is significant, but is ignored in most previous works. 

As an adversarial generation model, AT-GAN can generate adversarial examples very quickly. 
Here we evaluate the efficiency of each attack method to attack \textsf{Model A} on MNIST data as an example. The average time of generating \(1000\) adversarial examples is summarized in Table \ref{tab:time}. Among the five attack methods, AT-GAN is the fastest as it could generate adversarial examples without the target classifier and gradient calculation once it is trained. Note that \citet{Song2018Constructing} needs much longer time than others because it needs multiple searches and queries to generate one adversarial example.

\begin{table}[htb]
  \caption{Comparison on the example generation time, measured by generating 1000 adversarial instances using \textsf{Model A} on MNIST. }
   \vspace{0.5em} 
  \label{tab:time}
  \centering
  \begin{tabular}{cccccc}
    \toprule
    & FGSM & PGD & R+FGSM & Song's & AT-GAN\\
    \midrule
	Time &0.3s &1.8s &0.4s	& $\geq$ 2.5min &0.2s\\
    \bottomrule   
  \end{tabular}
\end{table}

\begin{figure*}[htb]
    \centering
    \subfigure[Adversarial examples for \textsf{Model A} on MNIST]{\includegraphics[scale=0.6]{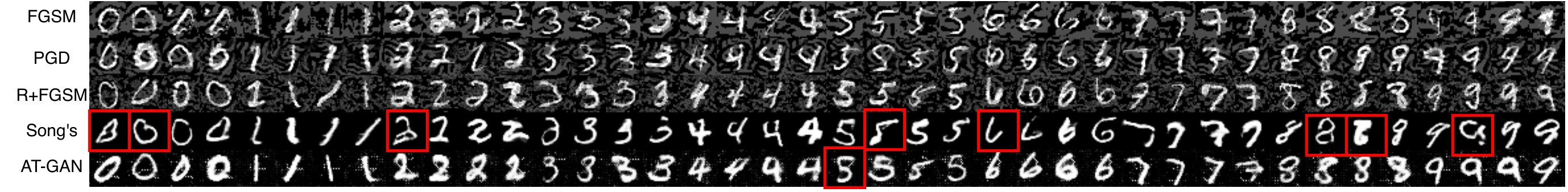}}
    \subfigure[Adversarial examples for \textsf{Model A} on Fashion-MNIST)]{\includegraphics[scale=0.6]{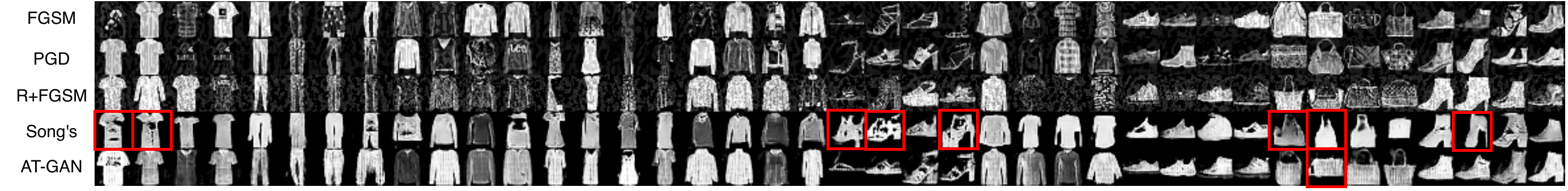}}
    \subfigure[Adversarial examples for \textsf{CNN} on CelebA]{\includegraphics[scale=0.48]{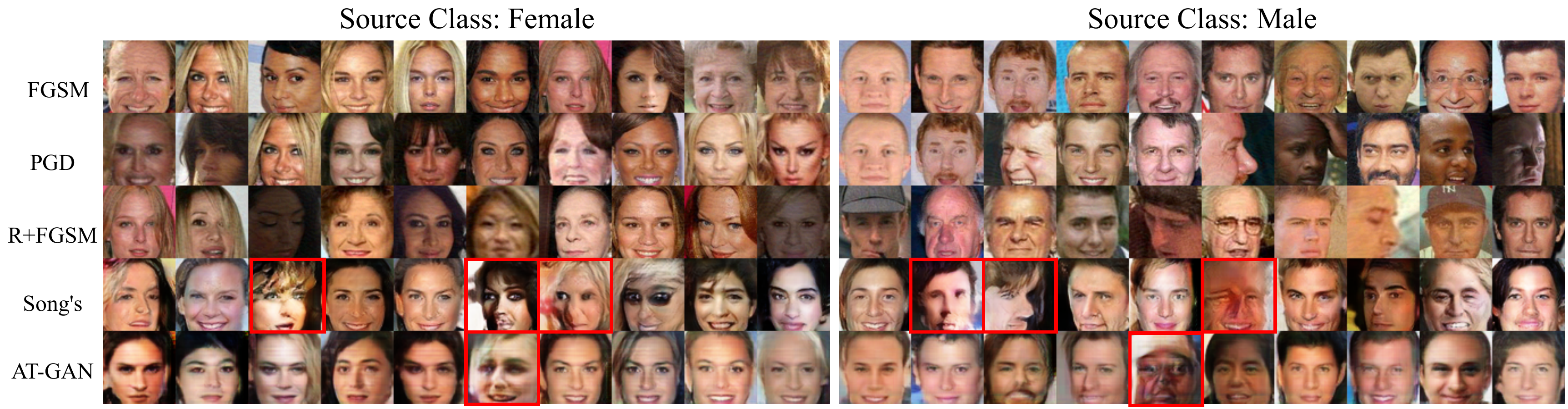}}
    \caption{Adversarial examples generated by various methods on three datasets (Zoom in for details). The red borders indicate unrealistic adversarial examples generated by Song's method or AT-GAN.}
    \label{fig:adv_img}
\end{figure*}
\vspace{0.5em}
\subsubsection{Visualization of Adversarial Examples}
\label{sec:adv_fig}
The goal of adversarial examples is to fool neural networks but not to fool humans.
Thus, we illustrate some adversarial examples generated by different attacks for \textsf{Molde A} on MNIST and Fashion-MNIST, and CNN on CelebA.

As shown in Figure \ref{fig:adv_img}, on MNIST,
AT-GAN generates slightly more realistic images than \citet{Song2018Constructing}, for example on “0” and “3”.  On both Fashion-MNIST and CelebA, some adversarial examples generated by Song’s method are not as realistic as AT-GAN to human perception, for example on “t-shirt/top (0) ”, “sandal (5)” and the details of some faces. 
As for perturbation-based attack, their adversarial examples are not clear enough especially on MNIST and Fashion-MNIST datasets due to the adversarial perturbations. More non-constrained adversarial examples generated by AT-GAN for target attack are shown in Appendix \ref{app:exp}.

In general, AT-GAN can generate realistic and diverse adversarial examples as Eq. \ref{eq:dis} forces the generated non-constrained adversarial examples to be close enough to the benign examples generated by the original generator.

\subsection{Transferability of AT-GAN}
One important feature for adversarial examples is the transferability across different models.
To demonstrate the transferability of non-constrained adversarial examples, we use adversarial examples generated by attacking \textsf{Model A} (MNIST and Fashion-MNIST) and \textsf{CNN} (CelebA), to evaluate the attack success rates on \textsf{Model C} (MNIST and Fashion-MNIST) and \textsf{VGG} (CelebA). 

As shown in Table \ref{tab:transfer}, non-constrained adversarial examples generated by AT-GAN exhibit moderate transferability although it is not always the best.
Note that as shown in Section \ref{sec:adv_fig}, adversarial examples generated by other attack methods are not as realistic as non-constrained adversarial examples and less realistic images could result in a higher transferability.
\begin{table*}[htb]
  \caption{Transferability of non-constrained adversarial examples and other traditional adversarial examples on the three datasets. For MNIST and Fashion-MNIST datasets, we attack \textsf{Model C} with adversarial examples generated on \textsf{Model A}. For CelebA dataset, we attack VGG using adversarial examples generated on CNN. Numbers represent the attack success rate (\%).}
  \label{tab:transfer}
  \vspace{1em}
  \centering
\scalebox{1}{
  \begin{tabular}{ccccccccccccc}
    \toprule
    & \multicolumn{4}{c}{MNIST} & \multicolumn{4}{c}{Fashion-MNIST} & \multicolumn{4}{c}{CelebA}\\
    \cmidrule(r){2-5} \cmidrule(r){6-9} \cmidrule(r){10-13}
    & Nor. & Adv. & Ens. & Iter. Adv. & Nor. & Adv. & Ens. & Iter. Adv. & Nor. & Adv. & Ens. & Iter. Adv.\\
    
    \midrule
    FGSM 	& 46.7 			& ~~4.2			& ~~1.7			& ~~4.6			& 68.9			& 23.1 			& 20.8 			& 14.8 & 15.6 & ~~4.3 & ~~3.3 & ~~4.1\\
	PGD 	& \textbf{97.5}	& ~~6.5			& ~~4.1			& ~~4.1			& \textbf{84.7}	& 27.6 			& \textbf{39.6}	& 14.6 & 18.3 & ~~4.3 & ~~3.1 & ~~4.1\\
	R+FGSM 	& 82.3 			& ~~6.7			& ~~4.8			& ~~4.1			& 21.2 			& 32.1 			& 17.5 			& 26.3 & 11.0 & ~~4.0 & ~~3.3 & ~~3.8\\
	Song's 	&  23.8		& 	20.8	&  20.6		& 		\textbf{20.1}	& 	39.2		& \textbf{34.0}	&  		31.5	& \textbf{30.3} & 9.6 & \textbf{31.8} & \textbf{21.5} & \textbf{38.8} \\
	AT-GAN 	& 65.3			& \textbf{24.6}	& \textbf{27.9}	& 17.2	& 58.0 			& 22.7 			& 32.0 			& 15.2 & \textbf{63.7} & 15.4 & 16.5 & 17.6\\
    \bottomrule   
  \end{tabular}
}
\end{table*}

\begin{figure*}[htbp]

\centering
\subfigure[Test set]{
\begin{minipage}[t]{0.16\linewidth}
\centering
\includegraphics[width=1.1in]{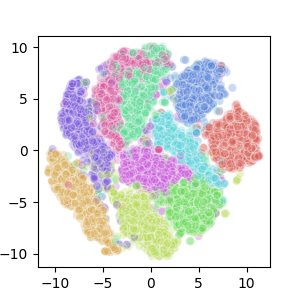}
\includegraphics[width=1.1in]{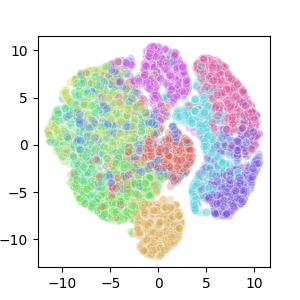}
\end{minipage}%
}%
\subfigure[FGSM]{
\begin{minipage}[t]{0.16\linewidth}
\centering
\includegraphics[width=1.1in]{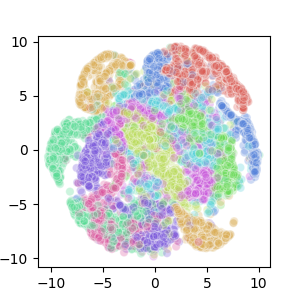}
\includegraphics[width=1.1in]{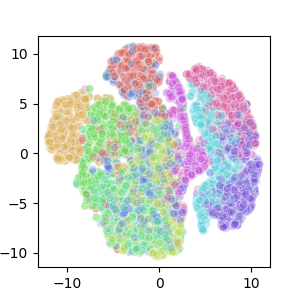}
\end{minipage}%
}%
\subfigure[PGD]{
\begin{minipage}[t]{0.16\linewidth}
\centering
\includegraphics[width=1.1in]{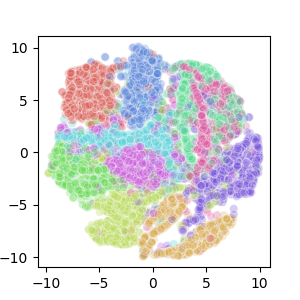}
\includegraphics[width=1.1in]{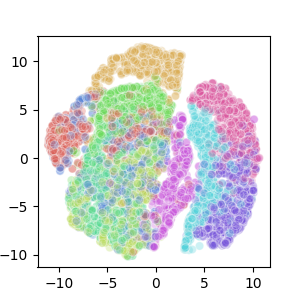}
\end{minipage}
}%
\subfigure[R+FGSM]{
\begin{minipage}[t]{0.16\linewidth}
\centering
\includegraphics[width=1.1in]{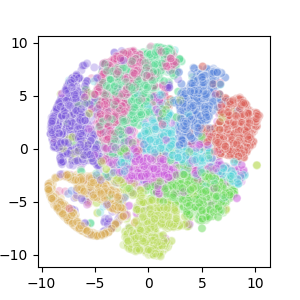}
\includegraphics[width=1.1in]{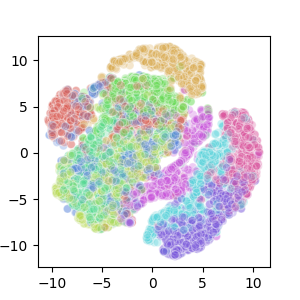}
\end{minipage}
}%
\centering
\subfigure[Song]{
\begin{minipage}[t]{0.16\linewidth}
\centering
\includegraphics[width=1.1in]{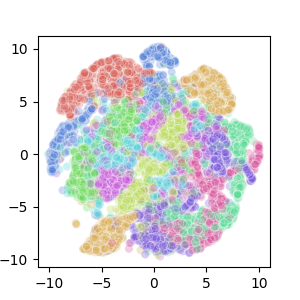}
\includegraphics[width=1.1in]{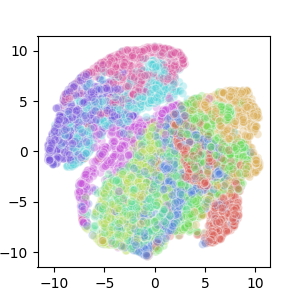}
\end{minipage}
}%
\centering
\subfigure[AT-GAN]{
\begin{minipage}[t]{0.16\linewidth}
\centering
\includegraphics[width=1.1in]{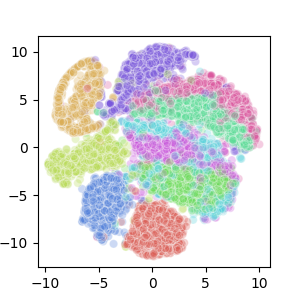}
\includegraphics[width=1.1in]{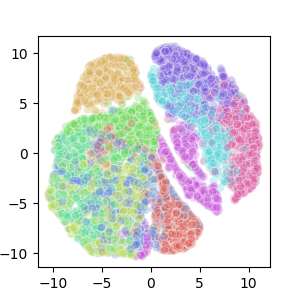}
\end{minipage}
}%
\centering
\caption{T-SNE visualizations for the combination of test set and adversarial examples generated by different adversarial attacks on MNIST (top) and Fashion-MNIST (bottom). For (a), we use 10,000 sampled real images in test set. For (b) to (f), we use 5,000 sampled images in test set and 5,000 adversarial examples generated by different attacks. The position of each class is random due to the property of t-SNE.}
\label{tab:tSNEResult}
\end{figure*}

\subsection{Visualization of Adversarial Distribution}
As discussed in Section \ref{sec:ana}, we give a brief analysis that AT-GAN can learn a distribution of adversarial examples which is close to the distribution of real data. To identify this empirically, we randomly choose $5,000$ benign images and $5,000$ adversarial examples generated by different attack methods, and merge these images according to their real label for MNIST and Fashion-MNIST. Then we use t-SNE \citep{Maaten2008tsne} on these images to illustrate the distributions in 2 dimensions.
T-SNE models each high-dimensional object in such a way that similar objects are modeled by nearby points and dissimilar objects are modeled by distant points with high probability.
It indicates that, if the adversarial examples have different distribution from that of the benign data, t-SNE could not deal well with them and the points with different categories will overlap with each other after dimension reduction, i.e., the results will be in chaos.

The results are illustrated in Figure \ref{tab:tSNEResult}. For AT-GAN, different categories are separated as that of the test set while those of other methods are mixed with each other.
It indicates the distribution that AT-GAN learned is indeed very close to the distribution of real data.

\section{Conclusion}
We propose an adversarial generator model which could generate non-constrained adversarial example for any input noise sampled from a distribution such as Gaussian distribution. Unlike perturbation-based adversarial examples that perturbs the input image or unrestricted adversarial examples that perturbs the input noise, our non-constrained adversarial examples are generated entirely from scratch without any constraint.

Specifically, we propose a novel attack framework called AT-GAN to estimate the distribution of adversarial examples 
that could be very close to the distribution of real data.
Extensive experiments and visualizations show that AT-GAN can generate diverse and realistic non-constrained adversarial examples efficiently. Besides, AT-GAN achieves higher attack success rates under white-box setting and exhibits moderate transferability.  

Our work of AT-GAN suggests that adversarial training, i.e., current strongest defense based on perturbation-based adversarial examples, could not guarantee the robustness against non-constrained adversarial examples.
A possible reason is that AT-GAN learns a more complete version on the distribution of adversarial examples, which is much more diverse than that of the perturbation-based method. 
Our method also offers a new way of building adversarial attacks by designing an adversarial generator directly, which may inspire more researches of this category in future work. 



\bibliography{example_paper}
\bibliographystyle{icml2020}

\newpage
~~
\newpage
\appendix
\section{Related Work}
\label{app:related_work}
\subsection{Gradient-based Attacks}
There are two types of attacks regarding how the attacks access the model. The \textit{white-box attack} can fully access the target model, while the \textit{black-box attack} \citep{Carlini2017Towards} has no knowledge of the target model. Existing black-box attacks mainly focus on \textit{transferability} \citep{Liu2017Delving,Bhagoji2017Exploring}, in which an adversarial instance generated on one model could be transferred to attack another model. 

We will introduce three typical adversarial attack methods. Here the components of all adversarial examples are clipped in $[0, 1]$.

\textbf{Fast Gradient Sign Method (FGSM).} FGSM \citep{Goodfellow2015Explaining} adds perturbation in the gradient direction of the training loss \(J\) on the input \(x\) to generate adversarial examples.
\begin{equation*}
x_{adv} = x + \epsilon \cdot sign(\nabla_x J(\theta,x,y))
\end{equation*}
Here \(y\) is the true label of a sample \(x\), \(\theta\) is the model parameter and \(\epsilon\) specifies the \(\ell_\infty \) distortion between \(x\) and \(x_{adv}\).

\textbf{Projected Gradient Descent (PGD).} PGD adversary~\citep{Madry2018Towards} is a multi-step variant of FGSM, which applies FGSM iteratively for \(k\) times with a budget \(\alpha\).
\begin{equation*}
    \begin{split}
        x_{adv_{t+1}} = \mathbf{clip}(x_{adv_t} &+ \alpha sign (\nabla_x J(\theta,x_{adv_t},y)), \\ &x_{adv_t} - \epsilon, x_{adv_t} + \epsilon)
    \end{split}
\end{equation*}
\begin{equation*}
    x_{adv_0} = x, \quad x_{adv} = x_{adv_k}
\end{equation*}
Here 
\(\mathbf{clip}(x, p, q)\) forces its input \(x\) to reside in the range of \([p, q]\).

\textbf{Rand FGSM (R+FGSM).} R+FGSM \citep{Florian2018Ensemble} first applies a small random perturbation on the benign image with a parameter \(\alpha\) (\(\alpha < \epsilon\)), then it uses FGSM to generate an adversarial example based on the perturbed image.
\begin{equation*}
\begin{split}
x_{adv} = x' + (\epsilon - \alpha) \cdot sign (\nabla_{x'} J(\theta, x', y)) \\
\text{where} \ x' = x + \alpha \cdot sign(\mathcal{N}(\mathbf{0}^d, \mathbf{I}^d))    
\end{split}
\end{equation*}

\subsection{Adversarial Training}
\label{sec:adv_training}
There are many defense strategies, such as detecting adversarial perturbations~\citep{Metzen2017Detecting}, obfuscating gradients~\citep{Buckman2018Thermometer,Guo2018Countering} and eliminating perturbations~\citep{Shen2017APEGAN, Liao2018Denoiser}, among which adversarial training is the most effective method~\citep{Athalye2018Obfuscated}. We list several adversarial training methods as follows.

\textbf{Adversarial training.} \citet{Goodfellow2015Explaining} first introduce the method of adversarial training, where the standard loss function \(f\) for a neural network is modified as:
\begin{equation*}
\tilde{J}(\theta,x,y)=\alpha J_f(\theta,x,y) + (1 - \alpha) J_f(\theta,x_{adv}, y).
\end{equation*}
Here \(y\) is the true label of a sample \(x\) and \(\theta\) is the model's parameter. The modified objective is to make the neural network more robust by penalizing it to count for adversarial samples. During the training, the adversarial samples are calculated with respect to the current status of the network. Taking FGSM for example, the loss function could be written as:
\begin{equation*}
\begin{split}
\tilde{J}(\theta,x,y)=&\alpha J_f(\theta,x,y) + (1 - \alpha) J_f(\theta,x+\\
&\epsilon sign(\nabla_x J(\theta,x,y)),y)    
\end{split}
\end{equation*}

\textbf{Ensemble adversarial training.} \citet{Florian2018Ensemble} propose an ensemble adversarial training method, in which DNN is trained with adversarial examples transferred from a number of fixed pre-trained models. 

\textbf{Iterative adversarial training.} \citet{Madry2018Towards} propose to train a DNN with adversarial examples generated by iterative methods such as PGD.

\section{Theoretical Analysis of AT-GAN}
\label{app:proof}
In this section, we provide the proof of theorems in Section \ref{sec:ana}.
\begin{thm}
Suppose $\max_{z,y} L_d(z,y) < \epsilon$, we have $KL(p_a \| p_g) \to 0$ when $\epsilon \to 0$.
\end{thm}
%
%
To prove this theorem, we first consider that given a distribution \(p(x)\) in space \(\mathcal{X}\), we construct another distribution \(q(x)\) by selecting the points \(p_\epsilon(x)\) in the \(\epsilon\)-neighborhood of \(p(x)\) for any \(x \in \mathcal{X}\). Obviously, when \(p_\epsilon(x)\) is close enough to \(p(x)\), \(q(x)\) has almost the same distribution as \(p(x)\). Formally, we can have the following lemma:
\begin{lemma}
\label{lem:dis}
Given two distributions \(P\) and \(Q\) with probability density function \(p(x)\) and \(q(x)\) in space \(\mathcal{X}\), if there exists a constant \(\epsilon\) that satisfies \(\|q(x) - p(x)\| < \epsilon\) for any \(x \in \mathcal{X}\), we could get \(KL(P\|Q) \to 0 \text{ when } \epsilon \to 0\).
\end{lemma}
\begin{proof}
For two distributions \(P\) and \(Q\) with probability density function \(p(x)\) and \(q(x)\), we could get \(q(x) = p(x) + r(x)\) where \(\|r(x)\| < \epsilon\).
\begin{equation*}
\begin{split}
&KL(P\|Q) = \int p(x) \log\frac{p(x)}{q(x)}dx \\
&= \int p(x)\log p(x)dx - \int p(x) \log q(x) dx \\
\end{split}
\end{equation*}

\begin{equation*}
\begin{split}
&= \int (q(x)-r(x))\log p(x)dx - \\
& \quad \int (q(x)-r(x)) \log q(x) dx \\
&= \int q(x) \log p(x) dx - \int q(x) \log q(x) dx \\
& \quad - \int r(x) \log p(x) dx + \int r(x) \log q(x) dx \\
&= \int r(x) \log \frac{q(x)}{p(x)}dx - KL(Q\|P) \\
&\leq \int \epsilon \log (1 + \frac{\epsilon}{p(x)}) dx
\end{split}
\end{equation*}
Obviously, when \(\epsilon \to 0\), we could get \(\int \epsilon \log (1 + \frac{\epsilon}{p(x)}) dx \to 0\), which means \(DL(P\|Q) \to 0\).
\end{proof}
Now, we come to proof of Theorem \ref{thm:atgan}.

\begin{proof}
For two distributions $p_a$ and $p_g$, $\max_{y,z} L_d < \epsilon$ indicates $\forall z \sim p_z, \|p_a(z)-p_g(z)\| < \epsilon$. According to Lemma \ref{lem:dis}, we have $KL(p_a \| p_g) \to 0$ when $\epsilon \to 0$. This concludes the proof.
\end{proof}. 

\begin{thm}
The global minimum of the virtual training of AC-WGAN\_GP is achieved if and only if \(p_g = p_{data}\).
\end{thm}
\vspace{-1em}
\begin{proof}
To simplify the analysis, we choose a category $y$ of AC-WGAN\_GP and denote $p_g(x|y)$ and $p_{data}(x|y)$ the distribution the generator learns and the distribution of real data respectively. Then for each category, the loss function is equivalent to WGAN\_GP. We refers to \citet{Samangouei2018DefenseGAN} to prove this property. The WGAN\_GP min-max loss is given by:
\begin{equation}
\begin{split}
&\min_G \max_D V(D,G) = \mathbb{E}_{x \sim p_{data}(x)}[D(x)] -\\
& \quad \mathbb{E}_{z \sim p_z(z)}[D(G(z))] - \lambda \mathbb{E}_{\hat{x} \sim p_{\hat{x}}(\hat{x})}[(\|\nabla_{\hat{x}}D(\hat{x})\|_2 - 1)^2]  \\
&= \int_x p_{data}(x)D(x)dx - \int_z p_z(z) D(G(z)) dz\\
& \quad - \lambda \int_{\hat{x}} p_{\hat{x}}(\hat{x}) [(\|\nabla_{\hat{x}}D(\hat{x})\|_2 - 1)^2] d\hat{x} \\
&= \int_x [p_{data}(x) - p_g(x)]D(x)dx \\
& \quad - \lambda \int_{\hat{x}} p_{\hat{x}}(\hat{x}) [(\|\nabla_{\hat{x}}D(\hat{x})\|_2 - 1)^2] d\hat{x} \\
\end{split}
\label{WGANGPloss}
\end{equation}
For a fixed \(G\), the optimal discriminator \(D\) that maximizes \(V(D, G)\) should be:
\begin{equation}
D_G^*(x)=\left\{\begin{array}{l}
1 \qquad \text{if }p_{data}(x) \geq p_g(x) \\
0 \qquad \text{otherwise}\\
\end{array} \right.
\label{OptimalD}
\end{equation}

According to Eq. \ref{WGANGPloss} and Eq. \ref{OptimalD}, we could get:
\begin{equation}
\begin{split}
&V(D,G) =\int_x [p_{data}(x) - p_g(x)]D(x)dx \\
& \quad - \lambda \int_{\hat{x}} p_{\hat{x}}(\hat{x}) [(\|\nabla_{\hat{x}}D(\hat{x})\|_2 - 1)^2] d\hat{x} \\
&= \int_{\{x|p_{data}(x)\geq p_g(x)\}} (p_{data}(x) - p_g(x))dx\\
& \quad - \lambda \int_{\hat{x}} p_{\hat{x}}(\hat{x}) d\hat{x} \\
&= \int_{\{x|p_{data}(x)\geq p_g(x)\}} (p_{data}(x) - p_g(x))dx - \lambda \\
\end{split}
\label{WGANGPD}
\end{equation}

Let \(\mathcal{X} = \{x|p_{data}(x)\geq p_g(x)\}\), in order to minimize Eq. \ref{WGANGPD}, we set \(p_{data}(x)=p_g(x)\) for any \(x \in \mathcal{X}\). Then, since both \(p_g\) and \(p_{data}\) integrate to 1, we could get:
\begin{equation*}
\int_{\mathcal{X}^c} p_g(x)dx = \int_{\mathcal{X}^c} p_{data}(x)dx.
\end{equation*}

However, this contradicts Eq. \ref{OptimalD} where \(p_{data}(x) < p_g(x)\) for \(x \in \mathcal{X}^c\), unless \(\mu(\mathcal{X} ^c) = 0\) where \(\mu\) is the Lebesgue measure. 

Therefore, for each category we have $p_g(x|y) = p_{data}(x|y)$ which means $p_g(x) = p_{data}(x)$ for AC-WGAN\_GP.
\end{proof}

\section{More Details on Experiments}
\label{app:exp}
\subsection{More Experimental Setup}
\label{sec:ExpDetails}
Here we describe the details about the experimental setup, including the model architectures and attack hyper-paramters.

\textbf{Model Architectures.}~\label{sec:architecture}
We first describe the neural network architectures used in experiments. The abbreviations for components in the network are described in Table \ref{tab:Abbreviation}. 
The architecture of AC-WGAN\_GP for MNIST and Fashion-MNIST is shown in Table \ref{GANDetails} where the generator and discriminator are the same as in \citet{Chen2016InfoGAN}, while the architecture of AC\_WGAN\_GP for CelebA is same as in \citet{Gulrajani2017WGANGP}. The details of \textsf{Model A} through \textsf{D} are described in Table \ref{ModelDetails}. Both \textsf{Model A} and \textsf{Model C} have convolutional layers and fully connected layers. The difference is only on the size and number of convolutional filters. \textsf{Model B} uses dropout as its first layer and adopts a bigger covolutional filter so that it has less number of parameters. \textsf{Model D} is a fully connected neural network with the least number of parameters and its accuracy will be lower than other models. VGG in our experiments is the same as VGG16 in \citet{Simonyan2015VGG} and we adopt ResNet from \citet{Song2018Constructing}.

\begin{table*}[htb]
  \caption{The abbreviations for network architectures.}
  \label{tab:Abbreviation}
  \vspace{1em}
  \centering
  \scalebox{0.95}{
  \begin{tabular}{cc}
    \toprule
   Abbreviation & Description\\
   \midrule
   Conv(\(m\text{, } k \times k\)) & A convolutional layer with \(m\) filters and filter size \(k\)\\
   DeConv(\(m\text{, } k \times k\)) & A transposed convolutional layer with \(m\) filters and filters size \(k\)\\
   Dropout(\(\alpha\)) & A dropout layer with probability \(\alpha\)\\
   FC(\(m\)) & A fully connected layer with \(m\) outputs\\
   Sigmoid & The sigmoid activation function\\
   Relu & The Rectified Linear Unit activation function\\
   LeakyRelu(\(\alpha\)) & The Leaky version of a Rectified Linear Unit with parameter \(\alpha\)\\
	Maxpool(\(k\text{,} s\)) & The maxpooling with filter size \(k\) and stride \(s\)	\\
   \bottomrule
  \end{tabular}
  }
\end{table*}


\begin{table*}[htb]
  \caption{The architecture of WGAN\_GP with auxiliary classifier.}
  \label{GANDetails}
  \vspace{1em}
  \centering
  \scalebox{0.95}{
  \begin{tabular}{ccc}
    \toprule
   Generator & Discriminator & Classifier\\
   \midrule
	FC(\(1024\)) + Relu & Conv(\(64\text{, } 4\times 4\)) + LeakyRelu(\(0.2\)) & Conv(\(32\text{, } 3\times3\)) + Relu\\
   FC(\(7\times 7 \times 128\)) + Relu & Conv(\(128\text{, } 4\times 4\)) + LeakyRelu(\(0.2\)) & pooling(\(2\text{, } 2\))\\
   DeConv(\(64\text{, } 4 \times 4\)) + Sigmoid & FC(\(1024\)) + LeakyRelu(\(0.2\)) & Conv(\(64\text{, } 3 \times 3\)) + Relu\\
   DeConv(\(1 \text{, } 4 \times 4\)) + Sigmoid & FC(1) + Sigmoid& pooling(\(2\text{, }2\))\\
   & & FC(\(1024\))\\
   & & Dropout(\(0.4\))\\
   & & FC(\(10\)) + Softmax\\
   \bottomrule
  \end{tabular}
  }
\end{table*}

\begin{table*}[htb]
  \caption{The architectures of \textbf{Models A} through \textbf{D} and CNN used for classification. After each model's name we put the number of parameters of that model.}
  \label{ModelDetails}
  \vspace{1em}
  \centering
  \begin{tabular}{ccccc}
    \toprule
   Model A (\(3\text{,}382\text{,}346\)) & Model B (\(710\text{,}218\)) & Model C (\(4\text{,}795\text{,}082\)) \\
   \midrule
	Conv(\(64\text{, } 5 \times 5\))+Relu & Dropout(\(0.2\)) & Conv(\(128\text{, } 3\times 3\))+Relu\\

   Conv(\(64\text{, } 5\times 5\))+Relu & Conv(~~\(64\text{, } 8\times 8\))+Relu & Conv(~~\(64\text{, } 3\times 3\))+Relu\\
   Dropout(\(0.25\)) & Conv(\(128\text{, } 6\times 6\))+Relu & Dropout(\(0.25\))\\
   FC(\(128\))+Relu & Conv(\(128\text{, } 5\times 5\))+Relu & FC(\(128\))+Relu\\
   Dropout(\(0.5\)) & Dropout(\(0.5\)) & Droopout(\(0.5\))\\
   FC(\(10\))+Softmax & FC(\(10\))+Softmax & FC(\(10\))+Softmax\\
   \midrule
   Model D (\(509\text{,}410\)) & CNN (\(17,066,658\)) &  \\
   \midrule
   \multirow{2}{*}{\(\begin{bmatrix}
   \text{FC(300)+Relu} \\
   \text{Dropout(0.5)}
  \end{bmatrix} \times 4\)}& Conv(\(32\text{, } 3 \times 3\))+Relu & &\\
   & Conv(\(32\text{, } 3 \times 3\))+Relu & &\\
   & Dropout(\(0.3\)) & &\\
   FC(\(10\)) + Softmax & Conv(\(64\text{, } 3 \times 3\))+Relu & &\\
   & Conv(\(64\text{, } 3 \times 3\))+Relu & &\\
   & Maxpool(\(2\text{, }2\)) + Dropout(\(0.3\)) & &\\
   & Conv(\(128\text{, } 3 \times 3\))+Relu & &\\
   & Conv(\(128\text{, } 3 \times 3\))+Relu & &\\
   & Maxpool(\(2\text{, }2\)) + Dropout(\(0.3\)) & &\\
   & FC(\(512\)) + Relu & &\\
   & Dropout(\(0.3\)) & &\\
   & FC(\(10\)) + Softmax & &\\
   \bottomrule
  \end{tabular}
\end{table*}

\textbf{Hyper-parameters.}\label{sec:hyperparamters}~
The hyper-parameters used in experiments for each attack method are described in Table \ref{Hyperparamter}.
\begin{table*}[htb]
  \caption{Hyper-paramters of different attack methods on the datasets.}
  \label{Hyperparamter}
  \vspace{1em}
  \centering
  \scalebox{0.9}{
  \begin{tabular}{ccccc}
    \toprule
   \multirow{2}{*}{Attack} & \multicolumn{3}{c}{Datasets} \\
\cmidrule(r){2-4}
    & MNIST & Fashion-MNIST & CelebA & Norm\\
   \midrule
   FGSM & \(\epsilon=0.3\) & \(\epsilon=0.1\) & \(\epsilon=0.015\) & \(\ell_\infty\)\\
   PGD	& \(\epsilon=0.3\text{, } \alpha=0.075\text{,  epochs}=20\) & \(\epsilon=0.1\text{, } \alpha=0.01\text{,  epochs}=20\) & \(\epsilon=0.015\text{, } \alpha=0.005\text{,  epochs}=20\) & \(\ell_\infty\)\\
   R+FGSM & \(\epsilon=0.3\text{, } \alpha=0.15\) & \(\epsilon=0.2\text{, } \alpha=0.1\) & \(\epsilon=0.015\text{, } \alpha=0.003\) & \(\ell_\infty\)\\
   Song's & \(\lambda_1=100 \text{, } \lambda_2=0\text{, epochs}=200 \)& \(\lambda_1=100 \text{, } \lambda_2=0\text{, epochs}=200 \) & \(\lambda_1=100 \text{, } \lambda_2=100\text{, epochs}=200 \) & N/A\\
   AT-GAN & \(\alpha=2\text{, } \beta=1\text{, epochs}=100\) & \(\alpha=2\text{, } \beta=1\text{, epochs}=100\) & \(\alpha=3\text{, } \beta=2\text{, epochs}=200\) & N/A\\
   \bottomrule
  \end{tabular}
  }
\end{table*}

\subsection{Non-constrained Adversarial Examples for Target Attack}
\label{AdvExamples}
Here we show some non-constrained adversarial examples generated by AT-GAN for target attack.
The results are illustrated in Figure \ref{fig:target_result} and \ref{fig:target_celeba}.
Instead of adding perturbations to the original images, AT-GAN transfers the generator so that the generated adversarial instances are not in the same shape of the initial examples (in diagonal) generated by the original generator.
\begin{figure*}[htb]
\subfigure[MNIST]{
\includegraphics[scale=0.88]{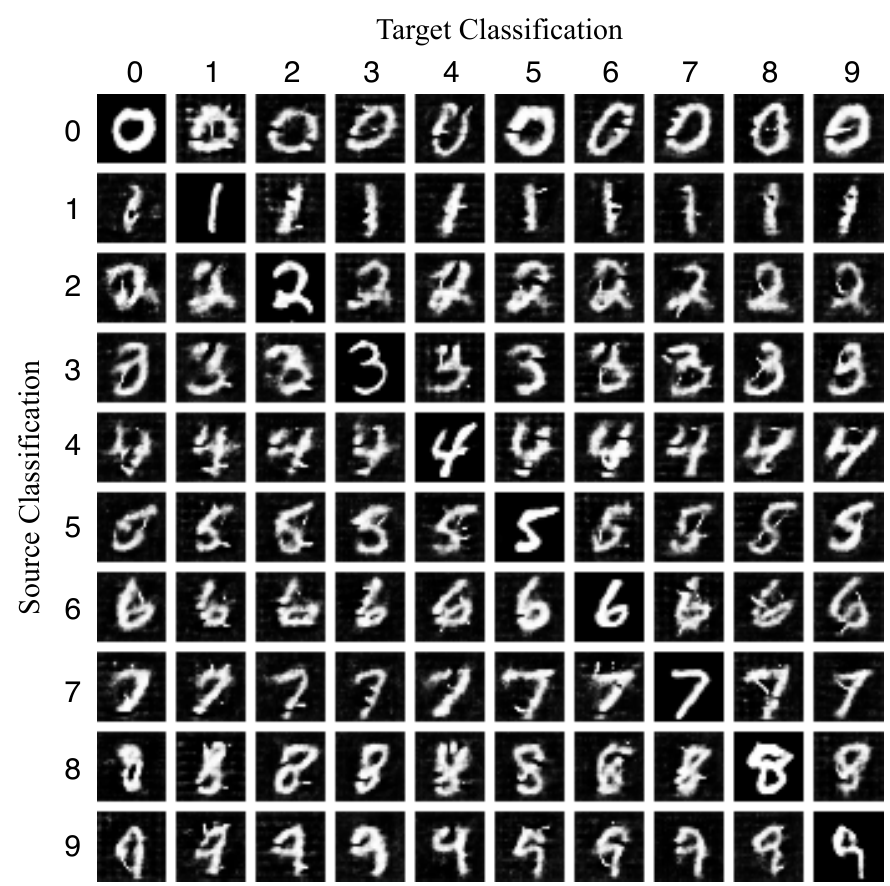}
}
\subfigure[Fashion-MNIST]{
\includegraphics[scale=0.88]{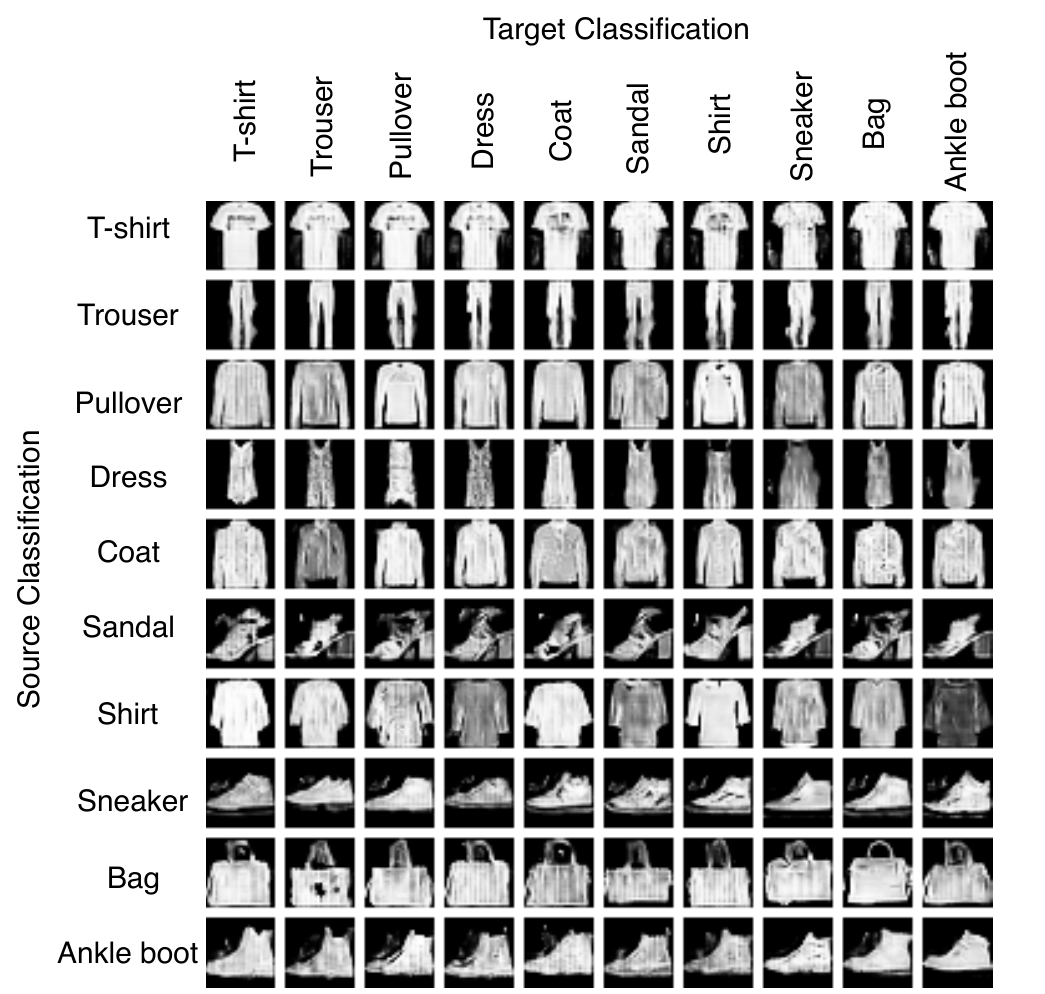}
}
\caption{Adversarial examples generated by AT-GAN to different targets with the same random noise input for each row. The images on the diagonal are generated by \(G_{original}\) which are not adversarial examples and treated as the initial instances for AT-GAN.}
\label{fig:target_result}
\end{figure*}
\begin{figure*}[htb]
    \centering
    \includegraphics[scale=0.5]{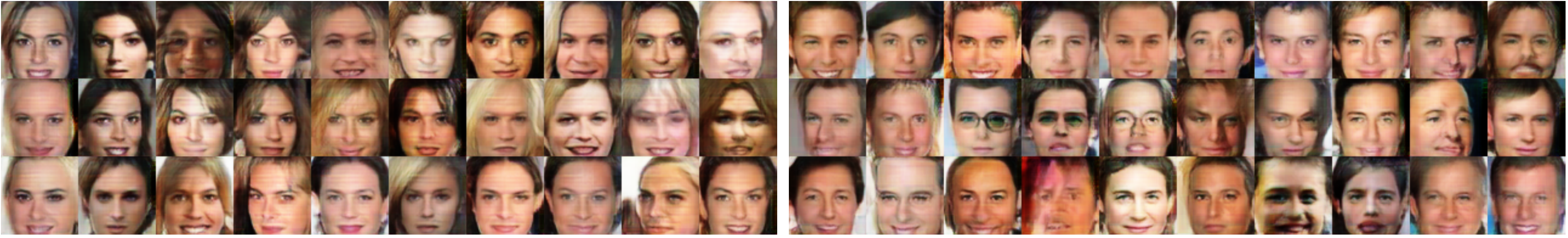}
    \caption{Adversarial examples generated by AT-GAN on CelebA dataset. }
    \label{fig:target_celeba}
\end{figure*}
\end{document}